\documentclass{article}

\usepackage{microtype}
\usepackage{graphicx}
\usepackage{subfigure}
\usepackage{booktabs} 
\newcommand{\bechmarkcount}{6 }
\newcommand{\blia}{binary passive label inference attack }
\newcommand{\liat}{LIA-Threshold }
\newcommand{\liam}{LIA-Delta-Margin }

\newcommand{\census}{CENSUS }

\newcommand{\fashinmnist}{FashionMNIST}
\newcommand{\msr}{Big-Vul}
\usepackage{comment}
\usepackage{hyperref}

\usepackage[accepted]{neurips_2023}

\usepackage{amsmath}
\usepackage{amssymb}
\usepackage{mathtools}
\usepackage{amsthm}

\usepackage[capitalize,noabbrev]{cleveref}

\theoremstyle{plain}
\newtheorem{theorem}{Theorem}[section]

\newtheorem{lemma}[theorem]{Lemma}

\theoremstyle{definition}
\newtheorem{definition}[theorem]{Definition}

\theoremstyle{remark}

\usepackage[textsize=tiny]{todonotes}

\icmltitlerunning{Passive Label Inference attack}

\begin{document}

\twocolumn[
\icmltitle{BLIA: Detect model memorization in binary classification model through passive Label Inference attack}



\icmlsetsymbol{equal}{*}

\begin{icmlauthorlist}
\icmlauthor{Mohammad Wahiduzzaman Khan}{yyy}
\icmlauthor{Sheng Chen}{yyy}
\icmlauthor{Ilya Mironov}{comp}
\icmlauthor{Leizhen Zhang}{yyy}
\icmlauthor{Rabib Noor}{yyy}

\end{icmlauthorlist}

\icmlaffiliation{yyy}{Department of Computer Science, University of Louisiana at Lafayette,LA,USA}
\icmlaffiliation{comp}{Meta, California, USA}

\icmlcorrespondingauthor{Mohammad Wahiduzzaman Khan}{mohammad-wahiduzzaman.khan1@louisiana.edu}
\icmlcorrespondingauthor{Firstname2 Lastname2}{first2.last2@www.uk}

\icmlkeywords{Machine Learning, ICML}

\vskip 0.3in
]




\begin{abstract}
Model memorization has implications for both the generalization capacity of machine learning models and the privacy of their training data. This paper investigates label memorization in binary classification models through two novel passive label inference attacks (BLIA). These attacks operate passively, relying solely on the outputs of pre-trained models, such as confidence scores and log-loss values, without interacting with or modifying the training process. By intentionally flipping 50\% of the labels in controlled subsets, termed "canaries," we evaluate the extent of label memorization under two conditions: models trained without label differential privacy (Label-DP) and those trained with randomized response-based Label-DP. Despite the application of varying degrees of Label-DP, the proposed attacks consistently achieve success rates exceeding 50\%, surpassing the baseline of random guessing and conclusively demonstrating that models memorize training labels, even when these labels are deliberately uncorrelated with the features.
\end{abstract}

\section{Intoduction}
\label{intro}
Deep learning has emerged as a powerful solution for challenges across a variety of fields, from healthcare diagnostics~\cite{ahmad2023revolutionizing} to financial risk analysis~\cite{chen2022cryptocurrency}. In recent years, deep learning in particular has fueled breakthroughs in areas such as autonomous driving~\cite{grigorescu2020survey}, recommendation systems~\cite{loukili2022machine}, and online advertising~\cite{zhao2021dear}. With applications spanning nearly every industry, these technologies continue to reshape our daily lives. 

\noindent Despite the remarkable performance of deep learning, understanding \emph{how} these models make decisions remains a significant challenge. While they are expected to learn generalizable patterns from training data, researchers have observed that deep learning architectures can \emph{memorize} specific features or examples rather than extracting broader insights~\cite{vitaly01,carlini2023extractingtrainingdatadiffusion,carlini2021extractingtrainingdatalarge,tirumala2022memorization}. This phenomenon presents two main concerns: 
\begin{enumerate}
\item It raises privacy and security risks~\cite{carlini2021extracting} by potentially exposing sensitive training data. For instance, in a two-party setup, an advertising network with access to user profiles and context might inadvertently memorize the advertiser’s private conversion events (labels) during training. If this memorized information is then probed, it could expose sensitive details about specific campaigns or user conversions, thereby compromising advertiser confidentiality. 
\item It can compromise a model's generalization ability~\cite{vitaly01}. For instance, a vulnerability detection model may memorize the specific labels (e.g., whether a given source code is classified as vulnerable or non-vulnerable, denoted as 0 or 1) rather than learning a generalized pattern to distinguish between vulnerable and non-vulnerable code. 
\end{enumerate} 

Memorization in machine learning generally takes two forms~\cite{wei2024memorizationdeeplearningsurvey}: \textit{example level memorization} that measures how much a model remembers individual training examples, indicating whether it can reconstruct or predict specific data points from the training set, and \textit{model-level memorization} that evaluates the overall tendency of a model to memorize training data across its entire parameter space, reflecting its global reliance on memorized patterns rather than generalization. A significant amount of research has been done in the first category~\cite{shokri2017membership,carlini2022membership,choquette2021label}. In this work, we focus on the latter, seeking to determine whether a given model memorizes its labels during the training process through a passive label inference attack.
 
Existing research on label inference attacks has primarily examined the extent to which trained models can leak private labels under normal circumstances. 
~\citet{aggarwal2021label} proposed a set of label inference attacks using log-loss scores to infer the labels.~\citet{fu2022label} exposed label privacy risks in vertical federated learning (VFL) by demonstrating how the bottom model structure and gradient updates enable an adversary to infer sensitive labels, even beyond the training set.In \cite{malek2021antipodes} authors measure Label-DP using canaries with deliberately mislabeled examples and applying black-box adversarial attacks to estimate empirical privacy loss. It contrasts these empirical results with theoretical guarantees provided by mechanisms like PATE and ALIBI, assessing both practical resilience and alignment with formal DP bounds.
\noindent \textbf{Our contribution.} In this work, we investigate label memorization through label inference attacks in two distinct settings: 
\begin{itemize} 
\item \textbf{Without Label-DP:} We conduct binary passive label inference attacks without incorporating any label differential privacy (Label-DP) techniques. 
\item \textbf{With Label-DP (Randomized Response):} We adopt a widely used Label-DP method---randomized response (RR) to preserve label privacy during model training. 
\end{itemize} 
In both scenarios, the primary objective is to detect the memorization of deliberately mislabeled examples.
The remainder of this paper is organized as follows. In Section~\ref{sec:background}, we provide the background and key definitions. Section~\ref{sec:problem-definition} and Section~\ref{sec:architecture} outline the formal problem definition and the architecture of our proposed approach, respectively.Section~\ref{sec:plia} provides a detailed exposition of the two passive label inference algorithms employed for the label inference attacks presented in this study. Section~\ref{sec:experiments} presents our experimental setup and results, and Section~\ref{sec:conclusion} concludes the paper by summarizing key findings and suggesting directions for future research.

\section{Background and Definitions}
\label{sec:background}
In this paper, we focus solely on binary classification problems, where the inputs are represented in the form $(X, Y)$. Here, $X$ denotes the feature matrix of size $\mathbb{R}^{m \times n}$, and $Y$ represents the label vector of size $\mathbb{R}^{m \times 1}$. In this notation, $m$ indicates the number of input rows, while $n$ refers to the number of features.
\begin{definition}
\label{def:inj}
\textbf{Differential Privacy(DP):} An algorithm $\mathcal{A}$ is $(\epsilon, \delta)$-differentially private if for any two neighboring datasets $D$ and $D'$ differing in at most one element, and for any subset of outputs $S \subseteq \text{Range}(\mathcal{A})$ ~\cite{dwork2006calibrating,dwork2014algorithmic}:

\[
\Pr[\mathcal{A}(D) \in S] \leq e^{\epsilon} \Pr[\mathcal{A}(D') \in S] + \delta,
\]

where $\epsilon \geq 0$ controls the privacy loss, it is also referred to as the privacy budget and $\delta$ accounts for the probability of privacy violation.
\end{definition}
\begin{definition}
   \textbf{ Label Differential Privacy(Label-DP):}  
 A learning algorithm $\mathcal{A}$ satisfies $(\epsilon, \delta)$-label differential privacy if, for any two datasets $D = \{(x_i, y_i)\}_{i=1}^m$ and $D' = \{(x_i, y'_i)\}_{i=1}^m$ that differ only in the label of a single example $i$, and for any subset $S \subseteq \text{Range}(\mathcal{A})$, the following holds:

\[
\Pr[\mathcal{A}(D) \in S] \leq e^{\epsilon} \Pr[\mathcal{A}(D') \in S] + \delta,
\]

where $\epsilon \geq 0$ is the privacy budget controlling the sensitivity to label changes.

\end{definition}
\textbf{Randomized Response (RR):} Randomized Response is a privacy-preserving technique that introduces controlled randomness to individual responses, enabling truthful data collection while safeguarding privacy~\cite{warner1965randomized}.
\begin{definition}
\textbf{Randomized response (RR):}
Given a binary-valued input $x \in \{0, 1\}$, the Randomized Response mechanism outputs:

\[
\mathcal{M}(x) =
\begin{cases}
x & \text{with probability } p, \\
1 - x & \text{with probability } 1 - p,
\end{cases}
\]

where $p = \frac{e^\epsilon}{1 + e^\epsilon}$ ensures that the mechanism satisfies $\epsilon$-differential privacy, with $\epsilon \geq 0$ controlling the level of privacy. A higher $\epsilon$ reduces privacy, while a lower $\epsilon$ increases privacy.
\end{definition}

\textbf{Label Memorization in Supervised Tasks:}
Label memorization~\cite{feldman2020does} quantifies how much a training algorithm \( A \) relies on a specific example \( (x_i, y_i) \) in dataset \( D \) by measuring the difference in the model’s ability to correctly predict \( y_i \) when \( (x_i, y_i) \) is included versus when it is removed (\( D \setminus i \)). Unlike example memorization, which involves storing detailed information about \( x_i \), label memorization primarily reflects the model's dependency on \( y_i \), making it susceptible to label inference but not necessarily data reconstruction attacks.

\textbf{Label Inference Attacks:} Label inference attacks aim to predict the true labels of data samples in a machine learning model by exploiting model outputs, such as probabilities or loss values~\cite{aggarwal2021label}, often revealing sensitive information.
\begin{definition}
\textbf{Label Inference Attack:}  
Given a training dataset $D = \{(x_i, y_i)\}_{i=1}^m$ and a model $f \leftarrow A(D)$, a label inference attack estimates the label $y_i$ by exploiting the model's behavior. 

\[
\hat{y}_i = \mathcal{P}(f(x_i))
\]

Here, the attacker uses the model's outputs (e.g., probabilities, logits, or confidence scores) to infer $\hat{y}_i$. The success of the attack depends on the model's memorization of the original or corrupted labels.

\end{definition}

    \textbf{Passive Label Inference Attack:}  
A passive label inference attack aims to infer the true labels of training data by observing the outputs of a pre-trained model, such as confidence scores or predictions, without interacting with or modifying the training process. This contrasts with active label inference attacks, where the adversary actively influences the training process~\cite{fu2022label} (e.g., by introducing modified inputs or querying specific examples) to extract label information.

\section{Problem Definition and Setting}
\label{sec:problem-definition}
In this work, we aim to design a binary passive label inference attack(BLIA) to detect and evaluate a model's propensity for label memorization. The proposed attack leverages a two-setting framework to systematically assess the extent to which a model memorizes labels during training. Below, we outline these settings and their implications for detecting label memorization.

\subsection{Settings and Methodology}

\paragraph{Setting 1: Without Label Differential Privacy (Without Label DP).}  
In the first setting, the model is trained on a dataset where the labels of a subset(canaries) of examples are randomly flipped with a 50\% probability. This randomized flipping ensures that the labels in this subset bear no relation to the corresponding feature vectors. After training, a label inference attack is performed to determine whether the model has memorized these flipped labels. 

\paragraph{Setting 2: Label Differential Privacy (Label DP) Applied.}  
In the second setting, similar to the first setting we first flip the labels of a subset of examples with a 50\% probability; we refer to this subset as canaries. Next, label differential privacy is applied to the entire training dataset using the randomized response \cite{warner1965randomized} mechanism, with varying levels of privacy to control the extent of label protection. Finally, after the model has been trained, a label inference attack is performed.

\subsection{Formal Problem Definition}

Let $D = D_1 \cup D_2$ represent the dataset, where $D_2$ consists of a subset(canaries) of examples $\{(x_i, y_i)\}$ with labels $y_i$ intentionally flipped with 50\% probability. We denote the randomized version of $D_2$ as $D_2' = \{(x_i, y_i')\}$. The training dataset is then $D' = D_1 \cup D_2'$. The model $f$ is trained on $D'$ using a training algorithm $\mathcal{A}$.

For a given input $x_i \in D_2'$, the label inference attack predicts the label $\hat{y}_i$ based on the model's output:
\[
\hat{y}_i = \mathcal{P}(f(x_i))
\]

The success ratio of the label inference attack is defined as:
\begin{equation}
\text{Success Ratio (SR)} = \frac{\text{\# Correctly inferred labels}}{\text{\# Total labels in } D_2'}
\end{equation}
This formula is justified because it quantifies the effectiveness of the attack by measuring the proportion of correctly inferred labels out of the total labels targeted. This standard accuracy metric provides a direct assessment of the inference attack’s success in predicting the true labels.

\subsection{Analysis and Implications}

If the label inference attack achieves a success ratio greater than 50\%:
\begin{itemize}
    \item \textbf{Setting 1 (Without Label DP):} The result indicates that the model memorizes the labels in $D_2'$, as the flipped labels have no meaningful relationship with the feature vectors.
    \item \textbf{Setting 2 (Label DP Applied):} Even with Label DP, a success ratio above 50\% conclusively demonstrates label memorization, as the privacy-preserving mechanism fails to prevent memorization of random labels.
\end{itemize}

By comparing the attack success rates across these settings, we can quantitatively evaluate the impact of label differential privacy on mitigating label memorization.

\begin{lemma}
Let \( D' \) be a dataset where 50\% of the labels are randomly flipped, making \( y_i' \) independent of \( x_i \). If a label inference attack achieves success rate \( SR > 0.5 \), the model \( f \) memorizes labels in \( D' \).
\end{lemma}

\begin{proof}
Consider a dataset \( D' \) where labels \( y_i' \) are flipped with probability 0.5. For binary classification, the probability of correctly inferring a label without using memorized information is:
\[
P(\text{success}) = 0.5.
\]

Assume \( f \) does not memorize \( D' \). Then, \( f(x_i) \) provides no information about \( y_i \) beyond random guessing, limiting any inference attack to \( S \leq 0.5 \).

If \( SR > 0.5 \), \( f(x_i) \) must encode information about \( y_i' \), implying \( f \) memorizes labels. This contradicts the assumption that \( f \) does not memorize.

Thus, \( SR > 0.5 \) indicates label memorization.
\end{proof}

\section{Architecture of BLIA}
\label{sec:architecture}
The architecture of the binary label inference attack is illustrated in Figure~\ref{icml-blia}. The process involves two main entities: the Challenger and the Attacker. The Challenger begins by splitting the dataset \(D\) into two subsets, \(D_1\) and \(D_2\). The labels in \(D_2\) are randomized to form a new subset \(D_2'\), resulting in a mix of original and modified labels. The Challenger then trains a model \(f\) using the combined dataset \(D_1 \cup D_2'\), producing prediction scores. The Attacker, who has access to the model \(f\) and the randomized subset \(D_2'\)(only features \(x_2)\), aims to infer the flipped labels by exploiting the model's output.

\begin{figure}[ht]
\vskip 0.2in
\begin{center}
\centerline{\includegraphics[width=\columnwidth]{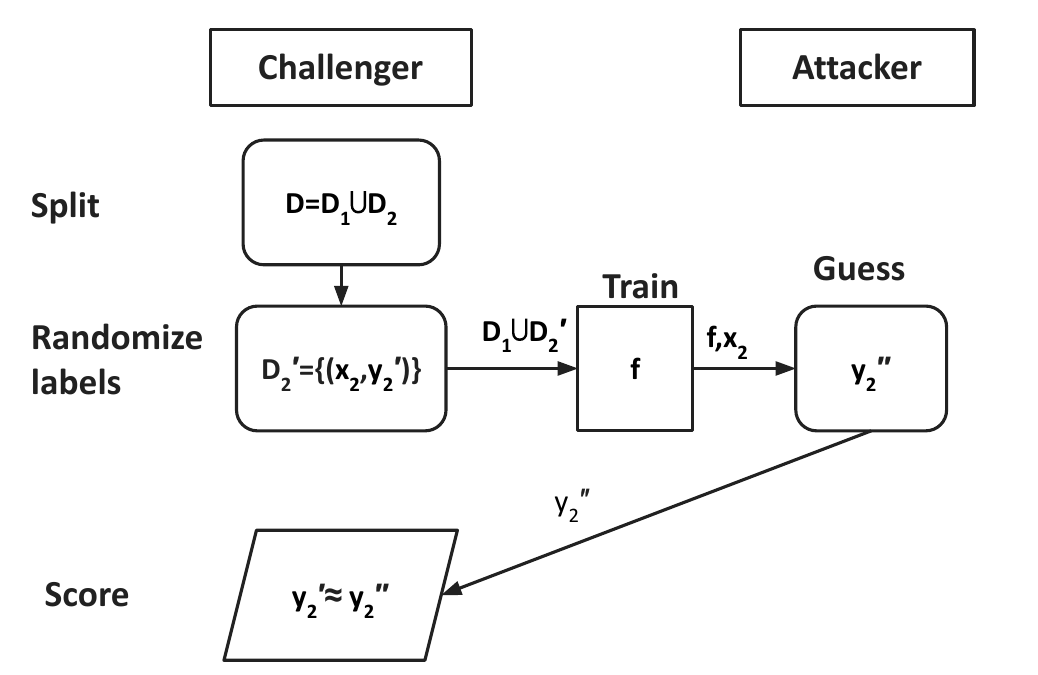}}
\caption{Binary label inference attack(BLIA) to detect model memorization}
\label{icml-blia}
\end{center}
\vskip -0.2in
\end{figure}
When Label-DP is applied, the architecture incorporates a randomized response mechanism to protect individual labels in the combined dataset \(D_1 \cup D_2'\). Specifically, each label in \(D_1 \cup D_2'\) is flipped with a certain probability determined by the privacy parameter \(\epsilon\), which controls the strength of differential privacy. Lower \(\epsilon\) values indicate stronger privacy protection, resulting in more frequent label flips. The model \(f\) is then trained on this perturbed dataset, reducing the attacker's ability to distinguish true labels from randomized labels.The success of the label inference attack indicates that the model retains information about flipped labels during the training process, demonstrating its tendency to memorize label perturbations introduced during training.

\section{Passive Label Inference Attack}
\label{sec:plia}
In this section, we describe our proposed passive label inference attacks, which is designed to exploit the outputs of pre-trained models to infer training labels. The attack focuses on utilizing the probabilities predicted by the model without interacting with or modifying the training process. 

\subsection{Attack Based on Threshold(\liat)}
This attack leverages the confidence scores produced by a model to infer the true labels of training examples. The adversary uses a pre-trained model to obtain predicted probabilities for each class and employs these probabilities to predict the labels based on a threshold, which can be fixed or dynamically computed.

The pseudocode for the probability-based label inference attack is presented in Algorithm~\ref{alg:binary_prob_attack}.

\begin{algorithm}[tb]
   \caption{Label Inference Attack using Fixed or Dynamic Thresholds}
   \label{alg:binary_prob_attack}
\begin{algorithmic}
   \STATE {\bfseries Input:} Predicted probabilities $\{\mathbf{p}_i = [p_{i,0}, p_{i,1}]\}_{i=1}^n$, Threshold type (``fixed'', ``mean'', ``median'')
   \STATE {\bfseries Output:} Inferred labels $\{\hat{y}_i\}_{i=1}^n$
   \STATE Initialize $\{\hat{y}_i\}_{i=1}^n \gets \emptyset$

   \STATE Extract confidence scores $\{c_i = p_{i,1}\}_{i=1}^n$
   \IF{\texttt{Threshold Type} = ``fixed''}
      \STATE $\tau \gets 0.5$
   \ELSIF{\texttt{Threshold Type} = ``mean''}
      \STATE $\tau \gets \text{mean}(\{c_i\}_{i=1}^n)$
   \ELSIF{\texttt{Threshold Type} = ``median''}
      \STATE $\tau \gets \text{median}(\{c_i\}_{i=1}^n)$
   \ENDIF

   \FOR{$i = 1$ {\bfseries to} $n$}
      \STATE $\hat{y}_i \gets 
      \begin{cases} 
      1, & \text{if } c_i \geq \tau, \\ 
      0, & \text{otherwise}.
      \end{cases}$
      \STATE Append $\hat{y}_i$ to $\{\hat{y}_i\}$
   \ENDFOR

   \STATE {\bfseries Return:} $\{\hat{y}_i\}_{i=1}^n$
\end{algorithmic}
\end{algorithm}

The threshold-based label inference attack is designed for binary classification tasks (\( K = 2 \)), where the model outputs confidence scores for two classes: \( p_{i,0} \) (class 0) and \( p_{i,1} \) (class 1). The adversary uses these scores to infer labels based on a threshold, which can be fixed or dynamically determined using the dataset's distribution.

\paragraph{Step-by-Step Processing:}
\begin{enumerate}
    \item \textbf{Adversary Setup:}  
    The adversary has access to a pre-trained binary classification model \( f \) that provides predicted probabilities \( \mathbf{p}_i = [p_{i,0}, p_{i,1}] \) for each input \( x_i \), where \( p_{i,0} + p_{i,1} = 1 \). 

    \item \textbf{Threshold Determination:}  
    The adversary determines the threshold \( \tau \) for label inference based on the selected method:
    \begin{itemize}
        \item \texttt{Fixed}: The threshold is set to \( \tau = 0.5 \).
        \item \texttt{Mean}: The threshold is computed as the mean of the confidence scores for the positive class, \( \tau = \text{mean}(\{p_{i,1}\}_{i=1}^n) \).
        \item \texttt{Median}: The threshold is computed as the median of the confidence scores for the positive class, \( \tau = \text{median}(\{p_{i,1}\}_{i=1}^n) \).
    \end{itemize}

    \item \textbf{Label Prediction:}  
    For each data point \( x_i \), the adversary infers the label \( \hat{y}_i \) as:
    \[
    \hat{y}_i =
    \begin{cases}
    1, & \text{if } p_{i,1} \geq \tau, \\
    0, & \text{otherwise}.
    \end{cases}
    \]

    \item \textbf{Output:}  
    After processing all data points, the adversary outputs the inferred labels \( \{\hat{y}_i\}_{i=1}^n \).
\end{enumerate}

\paragraph{Adversarial Insight:}  
This label inference attack uses a threshold-based approach to predict labels, allowing the adversary to adapt the inference strategy based on the dataset's distribution. The fixed threshold assumes a uniform decision boundary at \( \tau = 0.5 \), while dynamic thresholds (mean or median) leverage the underlying distribution of confidence scores to improve inference accuracy. By analyzing the inferred labels, the adversary can assess the model's tendency to leak information about its training data.

\subsection{Attack based on Delta-margin(\liam)}
This attack leverages the margins between the highest and second-highest confidence scores produced by a model, along with per-sample loss values, to infer the true labels of training examples. The adversary uses a pre-trained model to obtain predicted probabilities for each class and computes a score for each sample by combining the margin and loss, weighted by user-defined parameters. Labels are then inferred based on whether the computed score exceeds a threshold of zero.

The pseudocode for the delta-margin attack with loss adjustment is presented in Algorithm~\ref{alg:delta_margin_attack}.

\begin{algorithm}[tb]
   \caption{Delta-Margin Attack with Loss Adjustment}
   \label{alg:delta_margin_attack}
\begin{algorithmic}
   \STATE {\bfseries Input:} Predicted probabilities $\{\mathbf{p}_i = [p_{i,0}, p_{i,1}, \dots, p_{i,k-1}]\}_{i=1}^n$, Loss values $\{l_i\}_{i=1}^n$, Parameters $\alpha > 0$, $\beta > 0$
   \STATE {\bfseries Output:} Inferred labels $\{\hat{y}_i\}_{i=1}^n$
   \STATE Compute margins for each sample $i$:
   \[
   m_i = p_{i,1} - \text{SecondHighest}(\mathbf{p}_i),
   \]
   where $\text{SecondHighest}(\mathbf{p}_i)$ is the second largest probability in $\mathbf{p}_i$.
   \STATE Compute scores for each sample $i$:
   \[
   s_i = \alpha \cdot m_i - \beta \cdot l_i.
   \]
   \STATE Infer labels for each sample $i$:
   \[
   \hat{y}_i =
   \begin{cases}
   1, & \text{if } s_i > 0, \\
   0, & \text{otherwise}.
   \end{cases}
   \]
   \STATE {\bfseries Return:} $\{\hat{y}_i\}_{i=1}^n$
\end{algorithmic}
\end{algorithm}
The delta-margin attack is designed for binary classification tasks (\( K = 2 \)), where the model outputs confidence scores \( \mathbf{p}_i = [p_{i,0}, p_{i,1}] \) and per-sample loss values. The adversary combines the margin between the highest and second-highest probabilities with the loss values, weighted by parameters \( \alpha \) and \( \beta \), to compute a score for label inference.

\paragraph{Step-by-Step Processing:}
\begin{enumerate}
    \item \textbf{Adversary Setup:}  
    The adversary has access to a pre-trained binary classification model \( f \) that provides predicted probabilities \( \mathbf{p}_i = [p_{i,0}, p_{i,1}] \) for each input \( x_i \) and the corresponding loss values \( \{l_i\}_{i=1}^n \).

    \item \textbf{Margin and Score Computation:}  
    For each data point \( x_i \), the adversary computes:
    \begin{itemize}
        \item The margin: \( m_i = p_{i,1} - \text{SecondHighest}(\mathbf{p}_i) \).
        \item The score: \( s_i = \alpha \cdot m_i - \beta \cdot l_i \), where \( \alpha > 0 \) and \( \beta > 0 \) are weighting parameters.
    \end{itemize}

    \item \textbf{Label Prediction:}  
    The adversary infers the label \( \hat{y}_i \) as:
    \[
    \hat{y}_i =
    \begin{cases}
    1, & \text{if } s_i > 0, \\
    0, & \text{otherwise}.
    \end{cases}
    \]

    \item \textbf{Output:}  
    After processing all data points, the adversary outputs the inferred labels \( \{\hat{y}_i\}_{i=1}^n \).
\end{enumerate}

\paragraph{Adversarial Insight:}  
This attack exploits the difference between the top two confidence scores, adjusted by per-sample loss values, to infer labels. The use of parameters \( \alpha \) and \( \beta \) allows the adversary to balance the contribution of margins and losses. By analyzing the inferred labels, the adversary can evaluate the model's vulnerability to information leakage through confidence scores and loss values.

\section{Experiments}
\label{sec:experiments}
In this study, we evaluate a total of \bechmarkcount benchmarks for \blia, including \census~\cite{census_income_20}, \fashinmnist~\cite{xiao2017fashionmnistnovelimagedataset}, IMDB Sentiment Analysis~\cite{maas2011learning}, CIFAR-10, CIFAR-100~\cite{krizhevsky2009learning}, and a C/C++ code vulnerability dataset(\msr)~\cite{fan2020ac}. These benchmarks encompass both real-world datasets and academically standardized datasets frequently utilized in machine learning research. While datasets like FashionMNIST, CIFAR-10, and CIFAR-100 are inherently multi-class, we adapt them to binary classification settings to align with the scope of our work. Table~\ref{benchmark-table} provides an overview of benchmarks, training dataset sizes, the number of canaries, and model architectures used in the experiments. For each of the benchmarks we flipped 2\% of the training datasets label to generate the canaries for the attacker. Flipping only 2\% of the labels ensures minimal dataset perturbation, enhances the sensitivity of attacks to label memorization, reflects realistic privacy concerns, and avoids over-saturating the dataset with noise. Model architectures employed include Sequential, RNN, CNN, ResNet-18~\cite{resnet18}, and Transformer~\cite{fu2022linevul}, demonstrating a diverse set of methods tailored to the respective datasets. 
\begin{table}[t]
\caption{Overview of benchmarks, model architectures.}
\label{benchmark-table}
\begin{center}
\begin{small}
\begin{sc}
\begin{tabular}{lp{1.0cm}p{1.1cm}c}
\toprule
Benchmark     & Train size & Canaries size &  Architecture \\
\midrule
Census        & 26,048     & 520           & Sequential         \\
IMDB          & 25,000     & 500           & RNN                \\
FashionMNIST  & 60,000     & 1,200         & CNN                \\
CIFAR-10      & 50,000     & 1,000         & ResNet-18          \\
CIFAR-100     & 50,000     & 1,000         & ResNet-18          \\
\msr           &  150,908   & 3,018         & Transformer        \\
\bottomrule
\end{tabular}
\end{sc}
\end{small}
\end{center}

\end{table}

\subsection{BLIA without Label-DP}
In this section, we analyze the performance of label inference attacks on binary classification models in the absence of Label-DP.  Table~\ref{without-label-dp-acc} presents the accuracies for passive label inference attacks (LIA) without the application of Label Differential Privacy (Label-DP) across six benchmark datasets. The table reports three key metrics: training accuracy, \liat accuracy, and \liam accuracy. Notably, while the training accuracy remains consistently high for all benchmarks, the LIA threshold and delta margin accuracies exhibit considerable variability. 

As the canaries were flipped with a 50\% probability, achieving an accuracy greater than 50\% indicates successful inference. The results highlight varying levels of success across the benchmarks, demonstrating the effectiveness of label inference attacks under different settings.

\begin{table}[t]
\caption{Overview of passive label inference attack (LIA) accuracies without Label-DP.}
\label{without-label-dp-acc}
\begin{center}
\begin{small}
\begin{sc}
\begin{tabular}{lp{1.2cm}p{1.2cm}p{1.2cm}}
\toprule
Benchmark     & Train Accuracy & LIA Threshold Accuracy & LIA Delta Margin Accuracy \\
\midrule
CENSUS        & 84.72\%        & 51.73\%                & 51.15\%                   \\
IMDB          & 99.35\%        & 98.00\%                & 98.80\%                   \\
FashionMNIST  & 99.64\%        & 58.75\%                & 72.33\%                   \\
CIFAR-10      & 92.67\%        & 53.60\%                & 76.90\%                   \\
CIFAR-100     & 88.27\%        & 52.20\%                & 79.50\%                   \\
\msr      & 97.91\%        & 54.37\%                & 52.08\%                   \\
\bottomrule
\end{tabular}
\end{sc}
\end{small}
\end{center}
\end{table}

\noindent Across datasets, the results illustrate varying susceptibility to passive label inference attacks, with simpler benchmarks such as IMDB demonstrating substantially higher attack accuracies (e.g., 98.00\% for LIA threshold) compared to more complex datasets like \msr or CIFAR-100. This disparity highlights the significant role of dataset characteristics and inherent model architectures in determining vulnerability to label inference attacks. Interestingly, CIFAR-10, CIFAR-100 and FashionMNIST show a notable gap between LIA threshold and delta margin accuracies, with delta margin performing substantially better, emphasizing its potential as a more robust attack for evaluating label inference risks.

\subsection{BLIA with Label-DP}
This section investigates the performance of Binary Label Inference Attacks (BLIA) under the protection of Label Differential Privacy (Label-DP) across various benchmarks. The goal is to understand how Label-DP mitigates the effectiveness of BLIA by varying the privacy budget, \(\epsilon\), and analyzing its impact on attack accuracy and model performance.

Table~\ref{cifar100-lia-results} provides detailed results for the CIFAR-100 dataset, showing the train accuracy, \liat accuracy, and \liam across different \(\epsilon\) values. As \(\epsilon\) decreases, representing stronger privacy guarantees, a clear trade-off emerges: train accuracy declines while the success rates of LIA attacks (both threshold and delta margin) diminish. For instance, at \(\epsilon = \infty\), the delta margin accuracy reaches 79.50\%, indicating a significant vulnerability, whereas at \(\epsilon = 0.3\), the delta margin accuracy drops to 76.00\%, reflecting enhanced protection.

\begin{table}[t]
\caption{Results for CIFAR-100: Train Accuracy,  LIA Threshold Accuracy, and Delta Margin Accuracy for various $\epsilon$ values.}
\label{cifar100-lia-results}
\begin{center}
\begin{small}
\begin{sc}
\begin{tabular}{p{1.2cm}p{1.2cm}p{1.2cm}p{1.2cm}}
\toprule
Privacy budget $\epsilon$     & Train Accuracy & LIA Threshold Accuracy & LIA Delta Margin Accuracy \\
\midrule
Infinity       & 88.27               & 52.20                        & 79.50                      \\
4              & 88.55               & 52.00                        & 74.10                      \\
2              & 85.40               & 50.90                        & 54.40                      \\
1              & 82.86               & 51.30                        & 67.80                      \\
0.8            & 80.49               & 50.30                        & 66.70                      \\
0.5            & 78.80               & 51.10                        & 55.50                      \\
0.3            & 78.20               & 50.70                        & 76.00                      \\
\bottomrule
\end{tabular}
\end{sc}
\end{small}
\end{center}
\end{table}
Detailed results for other benchmarks, presented in a format similar to Table~\ref{cifar100-lia-results}, can be found in the Appendix section of this paper.

Figure \ref{fig:all_metrics} illustrates the relationship between train accuracy, delta margin attack success rate, and threshold-based attack success rate. This section provides an in-depth analysis of the results, highlighting the attack's efficacy and the relationship between privacy budgets and model vulnerability.
\begin{figure*}[ht]
    \centering
    \begin{tabular}{ccc}
        \includegraphics[width=0.3\textwidth]{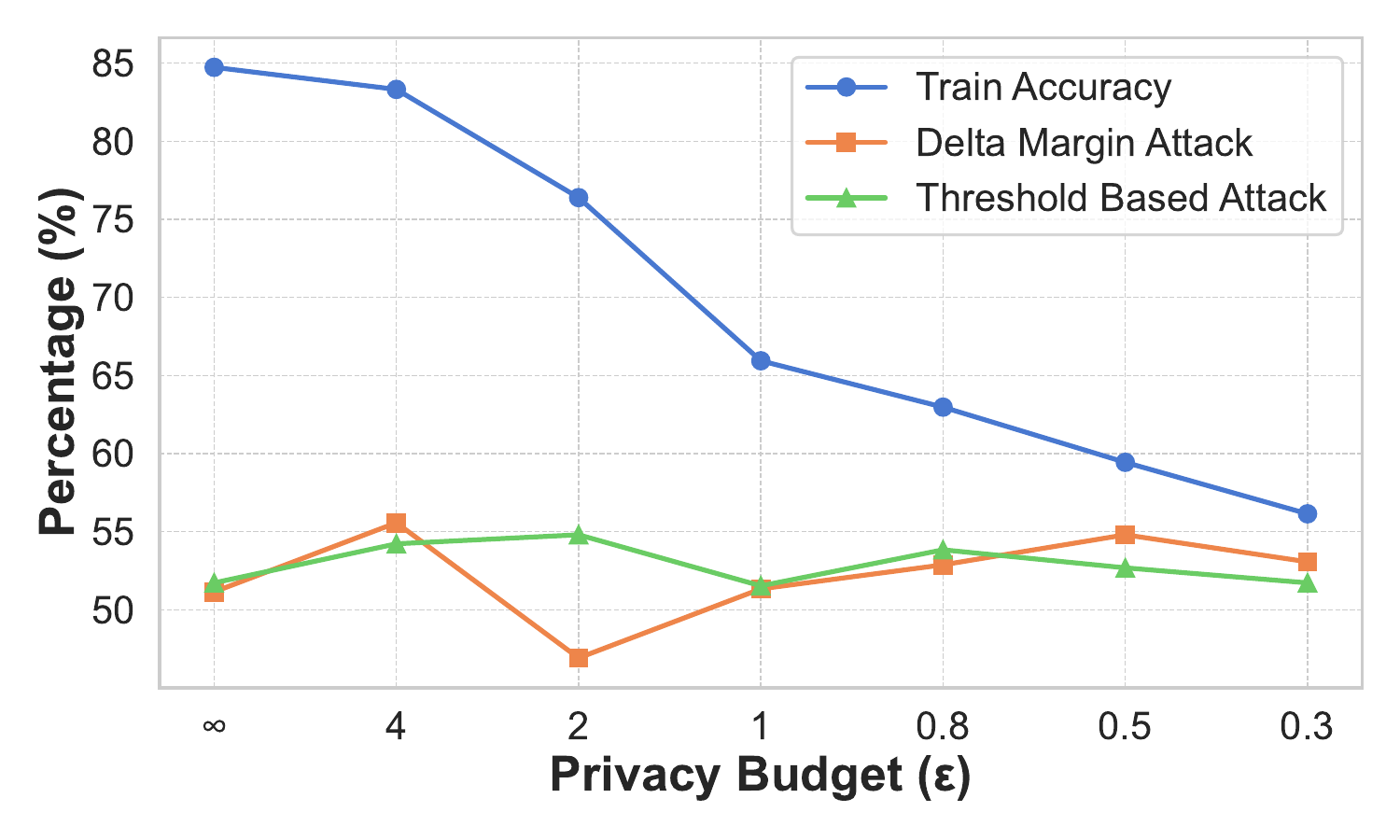} &
        \includegraphics[width=0.3\textwidth]{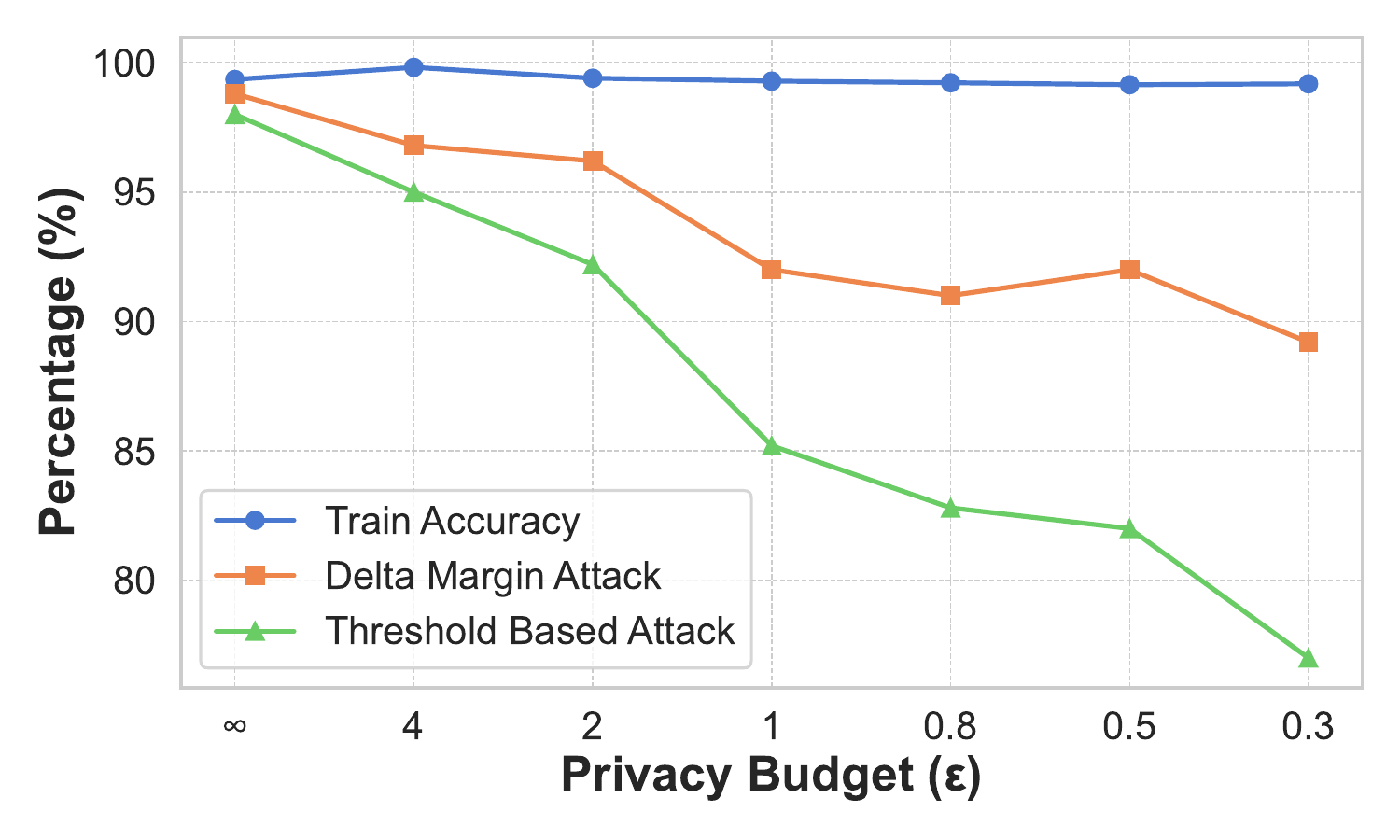} &
        \includegraphics[width=0.3\textwidth]{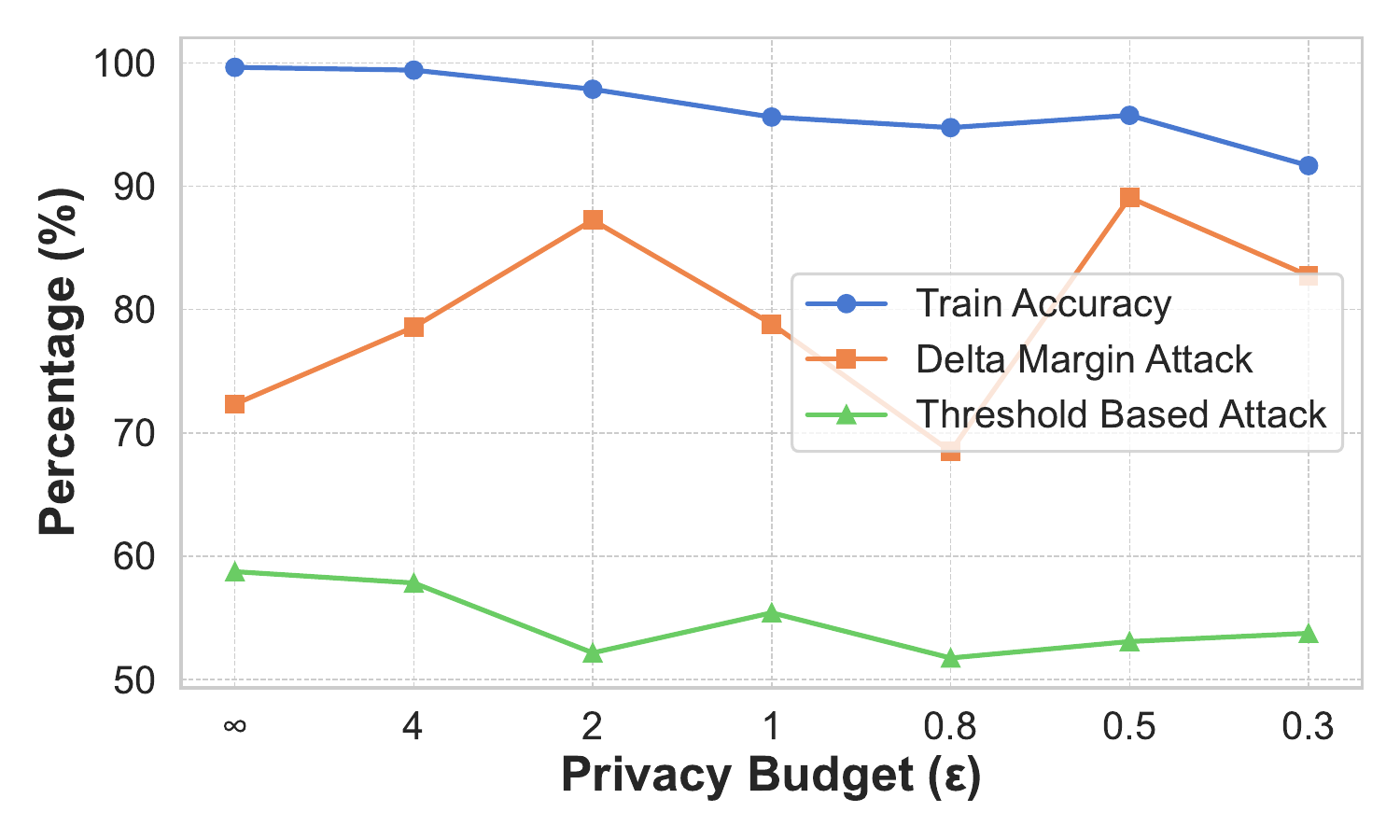} \\
        (a) Census & (b) IMDB & (c) FashionMNIST \\
        \includegraphics[width=0.3\textwidth]{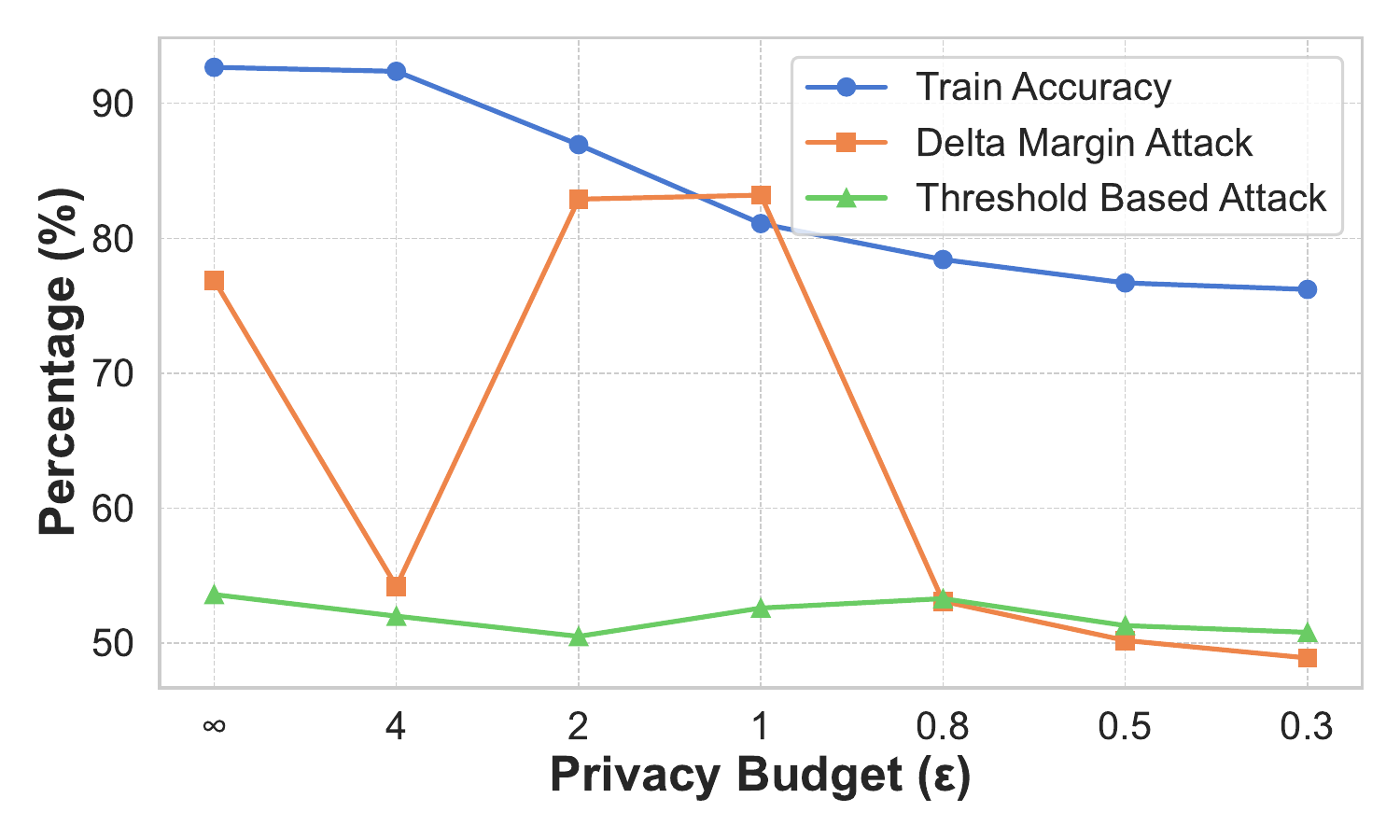} &
        \includegraphics[width=0.3\textwidth]{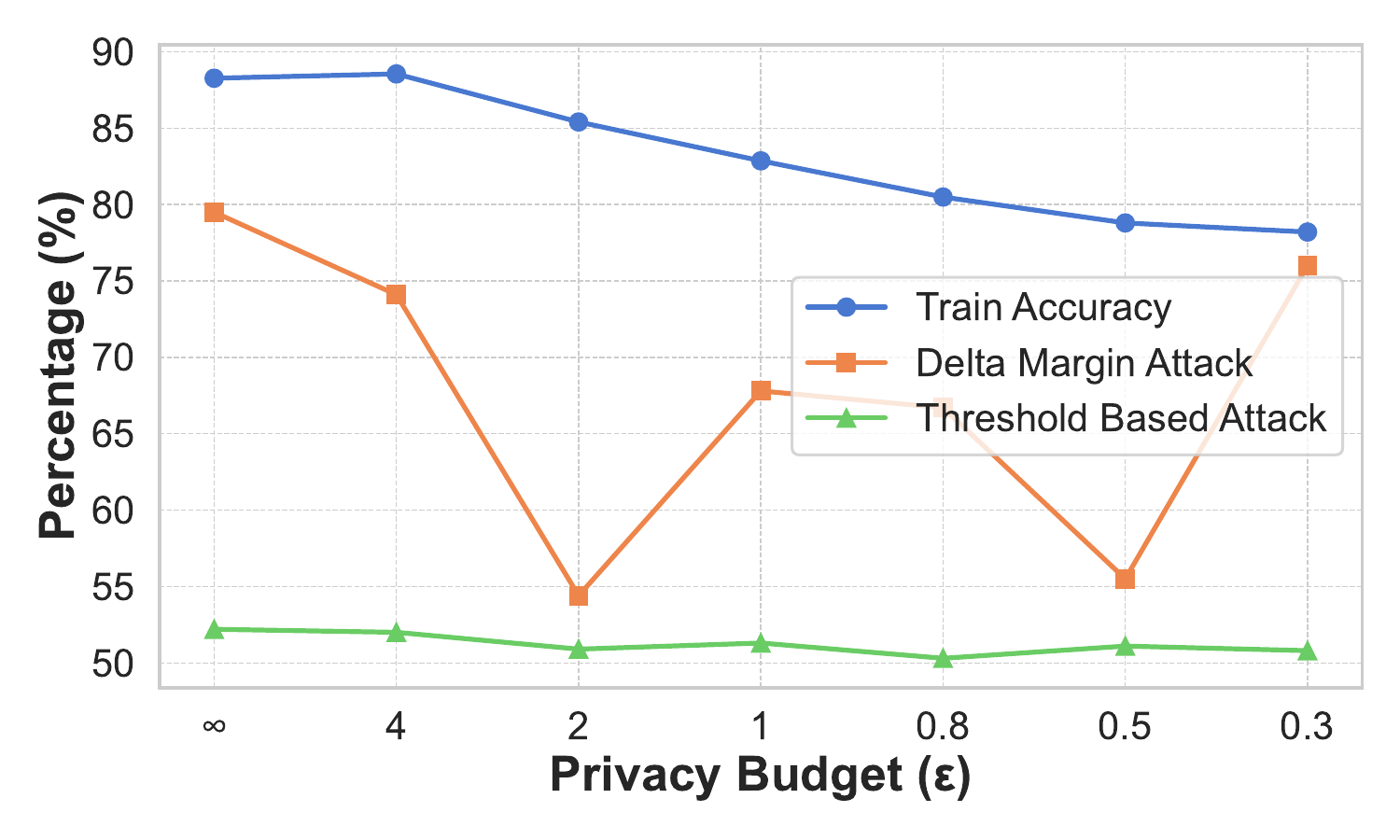} &
        \includegraphics[width=0.3\textwidth]{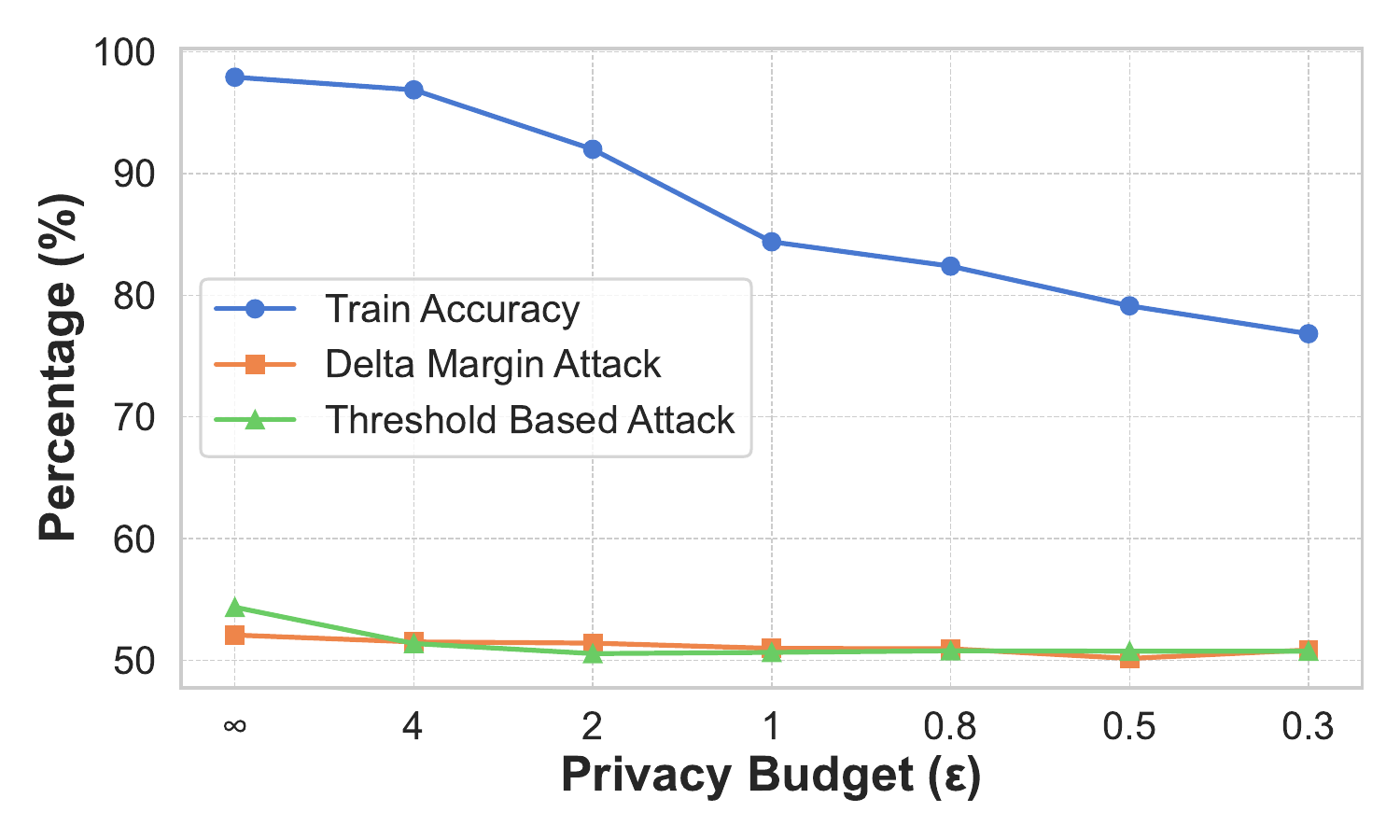} \\
        (d) CIFAR-10 & (e) CIFAR-100 & (f) MSR \\
    \end{tabular}
    \caption{Metrics vs. $\epsilon$ for all benchmarks. Panels (a)-(f) show train accuracy,threshold based attack success and delta margin attack success}
    \label{fig:all_metrics}
\end{figure*}

 Across all benchmarks, the delta margin attack exhibited a consistent trend of higher success rates compared to the threshold-based attack under varying privacy budgets. For higher privacy budgets (\( \epsilon = \infty \)), the attack success rates were notably higher, as the absence of noise in training allowed the attack to exploit model outputs more effectively. For instance, in the \texttt{IMDB} benchmark, the delta margin attack achieved a success rate of approximately 98.8\%, while in \texttt{CIFAR-100}, the rate reached 79.5\%.

As \( \epsilon \) decreased, introducing stronger noise for privacy preservation, the attack success rates generally declined. However, even for smaller privacy budgets such as \( \epsilon = 0.3 \), the delta margin attack maintained moderate effectiveness in several benchmarks. For instance, in \texttt{\fashinmnist}, it achieved a success rate of 82.75\%, suggesting that certain datasets remain more vulnerable under high privacy noise.

The performance of the label inference attack revealed distinct patterns across the datasets, indicating varying levels of vulnerability and resilience. Datasets with higher initial train accuracies under no privacy noise (\( \epsilon = \infty \)) generally exhibited higher attack success rates, suggesting a correlation between model confidence and susceptibility to attacks. For instance, benchmarks such as \texttt{IMDB} and \texttt{\fashinmnist} achieved near-perfect train accuracy with \( \epsilon = \infty \), making them more prone to successful label inference attacks.

As the privacy budget decreased (\( \epsilon \to 0.3 \)), the introduction of stronger noise led to a consistent drop in train accuracy and attack success rates across most datasets. However, the degree of this reduction varied. Datasets like \texttt{CENSUS} and \texttt{CIFAR-100} experienced more pronounced declines in train accuracy and attack success, indicating greater noise sensitivity. In contrast, datasets such as \texttt{\msr} and \texttt{IMDB} maintained relatively high train accuracies and attack success rates even under stringent privacy constraints, highlighting intrinsic dataset characteristics that make them less resilient to attacks.

Overall, the results suggest that datasets with simpler underlying patterns or higher class separability, such as \texttt{IMDB}, are more vulnerable to label inference attacks. Conversely, datasets with more complex structures or higher diversity, such as \texttt{CIFAR-10} and \texttt{CIFAR-100}, demonstrated increased resilience, particularly under lower privacy budgets.

\subsubsection{Comparison of Attack Techniques}

The \liam consistently outperformed the \liat variants across all benchmarks and privacy budgets. This superiority can be attributed to its ability to leverage confidence margins more effectively, particularly for larger privacy budgets (\( \epsilon \geq 1 \)).

\subsubsection{Impact of Privacy Budget}

The privacy budget \( \epsilon \) directly influenced both train accuracy and attack success rates. As \( \epsilon \) decreased, train accuracy declined, reflecting the trade-off between privacy and utility. For example, in the \texttt{CENSUS} benchmark, train accuracy dropped from 84.72\% (\( \epsilon = \infty \)) to 56.15\% (\( \epsilon = 0.3 \)). Concurrently, attack success rates diminished, showcasing the efficacy of privacy noise in mitigating vulnerabilities.

These results demonstrate that while smaller privacy budgets reduce attack success rates, certain datasets, such as \texttt{IMDB}, remain vulnerable even under strong privacy constraints. Among the evaluated attacks, the delta margin attack proved most effective, highlighting the need for advanced privacy-preserving techniques to safeguard against label inference vulnerabilities.

\section{Discussion}

Model memorization presents significant challenges for machine learning, impacting both generalization and privacy. Despite advances in privacy-preserving mechanisms, such as Label-Differential Privacy (Label-DP), the risk of label leakage through model outputs remains underexplored. Addressing this gap is critical, as it directly affects the secure deployment of machine learning systems in sensitive domains such as healthcare and finance. This work provides a systematic framework to evaluate label memorization, exposing vulnerabilities that persist even under strong privacy constraints.

While the notion of leveraging model outputs for label inference is conceptually simple, the results are significant for two reasons. First, they quantitatively demonstrate that model memorization persists across a wide range of architectures and datasets, even under randomized response-based Label-DP. Second, the success rates of our attacks exceeding 50\% provide empirical evidence that existing privacy mechanisms(Randomized Response) fail to mitigate label leakage effectively, revealing residual label memorization across a range of datasets and architectures and challenging the effectiveness of randomized response mechanisms.

While previous research has demonstrated that machine learning models can leak labels under specific conditions, our findings systematically reveal that \textbf{label memorization persists even under strong Randomized Response-based Label-DP}. Unlike~\cite{aggarwal2021label}, which uses log-loss values to infer labels without explicitly evaluating memorization, our work frames label inference as an empirical tool for detecting memorization. Similarly,~\cite{fu2022label} focus on label leakage in \textit{vertical federated learning}, where adversaries exploit gradient updates—whereas we show that memorization can be detected in a fully black-box setting using only confidence scores and log-loss values. Our study further supports concerns raised by~\cite{malek2021antipodes} regarding the limitations of \textit{Label-DP mechanisms} by demonstrating that even with strong privacy guarantees (low $\epsilon$), models still memorize labels beyond random guessing. 

\section{Limitations}

While this work presents a novel framework for passive label inference attacks and highlights their effectiveness, certain limitations should be acknowledged. First, the proposed methodology is primarily designed for binary classification tasks, and its generalization to multi-class scenarios remains unexplored. Extending the approach to multi-class settings would provide a broader understanding of its applicability.

Second, while this study demonstrates the vulnerabilities of models to passive label inference attacks under randomized response mechanisms, it does not evaluate more advanced privacy-preserving techniques such as PATE (Private Aggregation of Teacher Ensembles), ALIBI (Adaptive Label Inference-Based Intervention)\cite{malek2021antipodes}, or other differential privacy mechanisms\cite{abadi2016deep,ghazi2021deep} applied at the model or feature level. These advanced techniques could potentially offer stronger defenses against label memorization and warrant further investigation.

Lastly, the reliance on randomized label flipping to create canaries may not fully capture the complexity of more structured or correlated label perturbations. Such structured manipulations could influence the attack's success rate in real-world scenarios. Addressing these limitations would provide valuable insights and pave the way for future research in this domain.

\section{Conclusion}
\label{sec:conclusion}
This study highlights the vulnerability of binary classification models to label memorization through the proposed Binary Label Inference Attack (BLIA). By analyzing pre-trained model outputs, such as confidence scores and log-loss values, we demonstrated the feasibility of passive label inference attacks across multiple datasets, irrespective of the presence of label differential privacy (Label-DP). The findings underscore two critical points: models tend to memorize even deliberately corrupted labels, and existing Label-DP mechanisms, including randomized response, fail to completely mitigate this memorization.

These results hold significant implications for both privacy-preserving machine learning and the broader field of model generalization. The consistent success rates of the proposed attacks emphasize the need for more robust privacy-preserving techniques capable of balancing utility and privacy. Moreover, our framework can serve as a benchmark for evaluating the efficacy of emerging Label-DP methods.

Future research could explore extending these attacks to multi-class classification scenarios and evaluating their performance against advanced privacy-preserving algorithms. Additionally, investigating the interaction between data diversity, model architecture, and susceptibility to label memorization would provide deeper insights into enhancing model robustness.

\bibliography{main_paper}
\bibliographystyle{icml2025}

\newpage
\appendix

\section*{Appendix for BLIA: Detect Model Memorization in Binary Classification Model through Passive Label Inference Attack}
This appendix provides extended results for different privacy budgets ($\epsilon$) across multiple benchmarks, supplementing the findings reported in the main paper.

\section{Section \ref{sec:experiments} Results in Detail}

\begin{table}[t]
\caption{Results for Census: Train Accuracy, LIA Threshold Accuracy, and Delta Margin Accuracy for various $\epsilon$ values.}
\label{census-lia-results}
\begin{center}
\begin{small}
\begin{sc}
\begin{tabular}{p{1.2cm}p{1.2cm}p{1.2cm}p{1.2cm}}
\toprule
Privacy budget $\epsilon$     & Train Accuracy & LIA Threshold Accuracy & LIA Delta Margin Accuracy \\
\midrule
Infinity       & 84.72 & 51.73 & 51.15 \\
4              & 83.31 & 54.23 & 55.58 \\
2              & 76.39 & 54.81 & 46.92 \\
1              & 65.94 & 50.38 & 51.35 \\
0.8            & 62.98 & 50.77 & 52.88 \\
0.5            & 59.44 & 50.38 & 54.81 \\
0.3            & 56.15 & 50.38 & 53.08 \\
\bottomrule
\end{tabular}
\end{sc}
\end{small}
\end{center}
\end{table}
The results for the Census dataset, presented in Table~\ref{census-lia-results}, reveal a consistent decline in train accuracy as the privacy budget $\epsilon$ decreases. This reduction is expected due to the stronger noise introduced by Label-DP at smaller $\epsilon$ values. Interestingly, the Delta Margin Accuracy remains relatively stable compared to the LIA Threshold Accuracy, particularly at lower $\epsilon$, suggesting that the delta margin method is less sensitive to noise and provides more consistent attack performance.

Overall, the Delta Margin Attack achieves higher accuracies across most $\epsilon$ values, indicating its robustness. The relatively low LIA attack accuracies, even at $\epsilon = \infty$, highlight that the Census dataset is less vulnerable to label inference attacks compared to other benchmarks. This may be attributed to the dataset's characteristics and the simpler model architecture employed.
\begin{table}[t]
\caption{Results for IMDB: Train Accuracy, LIA Threshold Accuracy, and Delta Margin Accuracy for various $\epsilon$ values.}
\label{imdb-lia-results}
\begin{center}
\begin{small}
\begin{sc}
\begin{tabular}{p{1.2cm}p{1.2cm}p{1.2cm}p{1.2cm}}
\toprule
Privacy budget $\epsilon$ & Train Accuracy & LIA Threshold Accuracy & LIA Delta Margin Accuracy \\
\midrule
Infinity       & 99.35 & 98.00 & 98.80 \\
4              & 99.82 & 95.00 & 96.80 \\
2              & 99.40 & 92.20 & 96.20 \\
1              & 99.29 & 84.20 & 92.00 \\
0.8            & 99.22 & 82.40 & 91.00 \\
0.5            & 99.15 & 80.40 & 92.00 \\
0.3            & 99.18 & 75.60 & 89.20 \\
\bottomrule
\end{tabular}
\end{sc}
\end{small}
\end{center}
\end{table}
Table~\ref{imdb-lia-results} shows the results for the IMDB dataset. Train accuracy remains consistently high across all $\epsilon$ values, exceeding 99\%, reflecting the robustness of the RNN architecture on this benchmark. However, both the LIA Threshold and Delta Margin Attacks achieve significantly higher accuracies compared to other datasets, with Delta Margin Accuracy reaching 98.8\% at $\epsilon = \infty$.

This indicates a heightened vulnerability of the IMDB dataset to label inference attacks, likely due to the relatively simple decision boundary in binary sentiment classification tasks. The gap between Delta Margin and Threshold Accuracy is marginal, suggesting that both attack methods perform comparably well for this dataset. As $\epsilon$ decreases, attack performance diminishes, illustrating the efficacy of Label-DP.
\begin{table}[t]
\caption{Results for FashionMNIST: Train Accuracy, LIA Threshold Accuracy, and Delta Margin Accuracy for various $\epsilon$ values.}
\label{fashionmnist-lia-results}
\begin{center}
\begin{small}
\begin{sc}
\begin{tabular}{p{1.2cm}p{1.2cm}p{1.2cm}p{1.2cm}}
\toprule
Privacy budget $\epsilon$ & Train Accuracy & LIA Threshold Accuracy & LIA Delta Margin Accuracy \\
\midrule
Infinity       & 99.64 & 58.75 & 72.33 \\
4              & 99.41 & 57.83 & 78.58 \\
2              & 97.87 & 52.17 & 87.25 \\
1              & 95.61 & 55.42 & 78.83 \\
0.8            & 94.75 & 51.75 & 68.50 \\
0.5            & 95.75 & 52.75 & 89.08 \\
0.3            & 91.67 & 52.92 & 82.75 \\
\bottomrule
\end{tabular}
\end{sc}
\end{small}
\end{center}
\end{table}
The results for FashionMNIST, summarized in Table~\ref{fashionmnist-lia-results}, highlight an intriguing pattern. While train accuracy declines as $\epsilon$ decreases, the Delta Margin Attack exhibits significantly higher accuracies than the LIA Threshold Attack across all $\epsilon$ values. For example, at $\epsilon = 2$, Delta Margin Accuracy peaks at 87.25\%, demonstrating its effectiveness on this dataset.

The dataset's complexity, combined with the CNN's ability to learn intricate patterns, likely contributes to the resilience against LIA Threshold Attacks at lower $\epsilon$ values. However, the high attack success rates at larger $\epsilon$ values emphasize the need for privacy-preserving mechanisms for FashionMNIST.
\begin{table}[t]
\caption{Results for CIFAR-10: Train Accuracy, LIA Threshold Accuracy, and Delta Margin Accuracy for various $\epsilon$ values.}
\label{cifar10-lia-results}
\begin{center}
\begin{small}
\begin{sc}
\begin{tabular}{p{1.2cm}p{1.2cm}p{1.2cm}p{1.2cm}}
\toprule
Privacy budget $\epsilon$ & Train Accuracy & LIA Threshold Accuracy & LIA Delta Margin Accuracy \\
\midrule
Infinity       & 92.67 & 53.60 & 76.90 \\
4              & 92.37 & 47.40 & 54.20 \\
2              & 86.94 & 50.20 & 82.90 \\
1              & 81.09 & 48.60 & 83.20 \\
0.8            & 78.43 & 47.30 & 53.10 \\
0.5            & 76.69 & 51.00 & 50.20 \\
0.3            & 76.21 & 50.80 & 48.90 \\
\bottomrule
\end{tabular}
\end{sc}
\end{small}
\end{center}
\end{table}
Table~\ref{cifar10-lia-results} demonstrates that the CIFAR-10 dataset exhibits a sharp decline in train accuracy as $\epsilon$ decreases, which is expected given the increasing strength of Label-DP. Delta Margin Attack achieves the highest success rate of 83.20\% at $\epsilon = 1$, outperforming the LIA Threshold Attack at all $\epsilon$ values.
The dataset's inherent complexity, coupled with the ResNet-18 model's capacity to generalize well, likely contributes to the lower LIA accuracies compared to simpler benchmarks. These results emphasize that while CIFAR-10 is less vulnerable to label inference attacks than some other datasets, privacy guarantees provided by Label-DP remain critical.
\begin{table}[t]
\caption{Results for\msr: Train Accuracy, LIA Threshold Accuracy, and Delta Margin Accuracy for various $\epsilon$ values.}
\label{linevul-lia-results}
\begin{center}
\begin{small}
\begin{sc}
\begin{tabular}{p{1.2cm}p{1.2cm}p{1.2cm}p{1.2cm}}
\toprule
Privacy budget $\epsilon$ & Train Accuracy & LIA Threshold Accuracy & LIA Delta Margin Accuracy \\
\midrule
Infinity       & 97.91 & 54.37 & 52.09 \\
4              & 96.87 & 51.39 & 51.52 \\
2              & 91.99 & 50.00 & 51.42 \\
1              & 84.39 & 50.66 & 50.99 \\
0.8            & 82.39 & 49.40 & 50.96 \\
0.5            & 79.13 & 50.53 & 50.17 \\
0.3            & 76.84 & 50.33 & 50.86 \\
\bottomrule
\end{tabular}
\end{sc}
\end{small}
\end{center}
\end{table}
Table~\ref{linevul-lia-results} reveals that the Line-Vul dataset exhibits minimal differences in attack performance between Delta Margin and LIA Threshold Attacks across all $\epsilon$ values. This dataset consistently achieves high train accuracy, exceeding 97\% at $\epsilon = \infty$, while attack accuracies remain near 50\%, suggesting a limited vulnerability to label inference attacks.

The relative stability in attack success rates, even at large $\epsilon$ values, may be attributed to the complexity of the Transformer architecture and the dataset's unique characteristics. This suggests that model architecture plays a significant role in mitigating label inference risks alongside Label-DP.

\end{document}